\documentclass[letterpaper]{article} 
\usepackage{aaai25}  
\usepackage{times}  
\usepackage{helvet}  
\usepackage{courier}  
\usepackage[hyphens]{url}  
\usepackage{graphicx} 
\usepackage{natbib}  
\usepackage{caption} 
\usepackage{algorithm}

\usepackage[noend]{algpseudocode} 
\usepackage{subcaption}  
\usepackage{hhline}
\usepackage{multirow}
\usepackage{placeins}
\usepackage{amssymb}
\usepackage{amsmath}
\usepackage{amsthm}
\usepackage{booktabs}
\usepackage{xcolor}
\newcommand{\suhail}[1]{\textcolor{black}{#1}}

\newtheorem{theorem}{Theorem}

\DeclareMathOperator*{\argmin}{arg\,min}
\algdef{SE}[DOWHILE]{Do}{doWhile}{\algorithmicdo}[1]{\algorithmicwhile\ #1}%

\usepackage{newfloat}
\usepackage{listings}
\DeclareCaptionStyle{ruled}{labelfont=normalfont,labelsep=colon,strut=off} 
\floatstyle{ruled}
\newfloat{listing}{tb}{lst}{}
\floatname{listing}{Listing}
\title{Lazy Heuristic Search for Solving POMDPs with Expensive-to-Compute Belief Transitions}
\author{
    Muhammad Suhail Saleem, Rishi Veerapaneni, Maxim Likhachev
}
\affiliations{
    Carnegie Mellon University\\
    \{msaleem2, rveerapa, mlikhach\}@andrew.cmu.edu

}

\usepackage{bibentry}

\begin{document}

\maketitle

\begin{abstract}
Heuristic search solvers like RTDP-Bel and LAO* have proven effective for computing optimal and bounded sub-optimal solutions for Partially Observable Markov Decision Processes (POMDPs), which are typically formulated as belief MDPs. A belief represents a probability distribution over possible system states. Given a parent belief and an action, computing belief state transitions involves Bayesian updates that combine the transition and observation models of the POMDP to determine successor beliefs and their transition probabilities. However, there is a class of problems, specifically in robotics, where computing these transitions can be prohibitively expensive due to costly physics simulations, raycasting, or expensive collision checks required by the underlying transition and observation models, leading to long planning times. To address this challenge, we propose Lazy RTDP-Bel and Lazy LAO*, which defer computing expensive belief state transitions by leveraging Q-value estimation, significantly reducing planning time. We demonstrate the superior performance of the proposed lazy planners in domains such as contact-rich manipulation for pose estimation, outdoor navigation in rough terrain, and indoor navigation with a 1-D LiDAR sensor. Additionally, we discuss practical Q-value estimation techniques for commonly encountered problem classes that our lazy planners can leverage. Our results show that lazy heuristic search methods dramatically improve planning speed by postponing expensive belief transition evaluations while maintaining solution quality. 
\end{abstract}

%

\section{Introduction}


Partially Observable Markov Decision Processes (POMDPs) provide a principled framework for decision-making under uncertainty and have been widely employed in various robotics domains, including autonomous navigation and manipulation tasks \cite{pomdp_driving, pomdp_manipulation}. As the true state of the system is not directly observable, POMDPs maintain a probability distribution over possible system states at any given time, referred to as belief states. They are commonly formulated and solved as a Markov Decision Process (MDP) in the belief space. However, solving POMDPs is notoriously challenging due to the belief space growing exponentially with each additional state variable (curse of dimensionality), as well as with the planning horizon (curse of history) \cite{kaelbling1998planning}.

Search-based solvers like RTDP-Bel \cite{rtdp-bel} and LAO* \cite{lao} leverage heuristics to focus computational efforts on promising regions of the belief space, efficiently converging to optimal or bounded suboptimal solutions. Although powerful and effective in many problems, their applicability still remains limited—particularly in real-world robotics problems. One of the many reasons for this is the inherent assumption that belief transitions are readily available. Computing transitions in the belief space, i.e., determining the successor beliefs and their probabilities given the current belief over the system state and an action, involves Bayesian updates that combine the transition and observation models of the underlying POMDP. However, in many robotics applications, this computation poses a significant bottleneck due to the high computational cost of querying these models. For instance, outdoor robot navigation may require high-fidelity physics simulations to predict action outcomes, while contact-rich manipulation tasks often involve expensive collision checks to model expected interactions. These operations are expensive to query, even for a single system state. Since the belief state is a distribution over the system states, computing the transitions from a belief state would involve querying these operations for \textit{every} state in (the support of) the belief. Since in a typical problem we can expect to search over thousands of belief states, the total cost of computing belief transitions becomes prohibitively expensive, limiting the practicality of POMDP-based solutions for such applications.


To mitigate this challenge, we propose lazy heuristic search solvers for POMDPs, which defer expensive belief state transition computations until necessary. Specifically, we introduce two novel algorithms: Lazy RTDP-Bel and Lazy LAO*. These methods leverage $Q$-value estimation to identify the most promising action from a belief state, and only the transition corresponding to the most promising action is computed. The belief state transitions for other actions are postponed until the search deems these actions necessary, i.e., have the potential to improve the current solution. By postponing these expensive evaluations and only computing them when needed, we save significant computation time. We show that the $Q$-value estimates leveraged by our lazy planners are equivalent to heuristic functions and are readily available for common problems, particularly in robotics. Moreover, if the $Q$-value estimates are conservative, then the proposed lazy planners, like their vanilla counterparts, are guaranteed to converge to the optimal solution.



We demonstrate the effectiveness of our proposed lazy planners on three robotics problems i) Manipulation for pose estimation using contacts, ii) Indoor navigation with a 1D LiDAR sensor, and iii) Outdoor rough terrain navigation. Belief transition computations for each of these domains are expensive due to them requiring mesh-to-mesh collision checks, ray casting operations, and physics-based simulations, respectively. Results presented in Section \ref{sec: results}, demonstrate the superior performance of the lazy algorithms on the abovementioned problems, significantly improving planning efficiency while maintaining solution quality.

\section{Background}

In this manuscript, we introduce the concept of laziness in the context of Goal-POMDPs \cite{bertsekas2012dynamic}, though it naturally extends to reward-based formulations. 

\subsection{Goal-POMDPs}

Formally, a discrete Goal-POMDP is defined by the tuple $\langle \mathcal{S}, \mathcal{A}, \mathcal{Z}, G, \mathcal{T}, \mathcal{O}, c, b_0 \rangle$, where $\mathcal{S}$, $\mathcal{A}$, and $\mathcal{Z}$ are finite non-empty sets representing the state space, action space, and observation space respectively. $G \subseteq \mathcal{S}$ is the set of goal states. The transition model $\mathcal{T}(s, a, s') = p(s'| s, a)$ defines the probability of transitioning to state $s'$ when executing action $a$ from state $s$, while the observation model $\mathcal{O}(s', a, z) = p(z| s', a)$ gives the probability of observing $z$ upon entering $s'$ via action $a$. The cost function $c(s, a)$ determines the cost of executing action $a$ from state $s$, and $b_0$ is the initial belief state, i.e., the distribution over possible initial states. We assume that goal states $g \in G$ are cost-free, absorbing, and fully observable ($z_g \in \mathcal{Z}$ denotes the observation made from a goal state) : (i) $c(g, a) = 0 , \forall , a \in \mathcal{A}$, (ii) $\mathcal{T}(g, a, g) = 1$, and (iii) $\mathcal{O}(g, a, z_g) = 1$ for all $a \in \mathcal{A}$, with $\mathcal{O}(s, a, z_g) = 0$ for $s \notin G$.

POMDPs are commonly solved by reformulating them as completely observable MDPs over the belief space. A belief state $b$ is a probability distribution over states, where $b(s)$ represents the probability of being in state $s$. Although the system state is not directly observable, we can compute belief updates that capture the effects of actions and observations on the belief state. The belief update after executing action $a$ from belief $b$ is given by

\begin{equation}
    b_a(s) = P(s| b, a) = \sum_{s' \in \mathcal{S}} \mathcal{T}(s', a, s) b(s')
\end{equation}
Similarly, the probability of observing $z$ after taking action $a$ from $b$ is, 

\begin{equation}
    b_a(z) = P(z | b, a) = \sum_{s \in \mathcal{S}} b_a(s) \mathcal{O}(s, a, z)
\end{equation}

Then, we can write the belief update when executing $a$ from $b$ and observing $z$ as
\begin{equation}
    b_a^z(s) = \mathcal{O}(s, a, z) b_a(s) / b_a(z) \quad \text{if } \, b_a(z) \neq 0
\end{equation}


This formulation transforms the POMDP into a fully observable belief MDP, where the transition function is given by $P(b_a^z | b, a) = b_a(z)$. Consequently, reaching a goal state in the original POMDP translates to reaching a goal belief $b_g$, where $b_g(s) = 0$ for all $s \notin G$. The objective of the problem is to compute a belief state policy from the initial belief $b_0$ to a goal belief $b_g$ while minimizing the expected cost. The Bellman equation for the belief MDP is given by
\begin{equation}
    V^*(b) = \min_{a \in \mathcal{A}} \,\, \Biggl\{ c(b, a) + \sum_{z \in Z} b_a(z) V^*(b_a^z) \Biggr\}
\end{equation}
for non-goal beliefs and $V^*(b_g) = 0$ for goal beliefs. $c(b, a)$ is the expected cost of executing action $a$ from $b$, i.e., $c(b, a) = \sum_{s \in S} c(s, a) b(s)$. 

\subsection{RTDP-Bel}
\begin{algorithm}[t]
\caption{\textsc{RTDP-Bel}} \label{Alg: RTDP-Bel}
\begin{algorithmic}[1]
\While {\textsc{Not Converged}}
\State $b = b_{0}$ 
\State \textbf{Sample} state $s$ with probability $b(s)$
\While {$b$ is not a goal belief}
\State \textbf{Evaluate} each action $a \in \mathcal{A}$ from $b$
\State \textbf{Compute} value of each action as:
    \begin{equation*}\vspace{-0.15cm}
    Q(b, a) = c(b, a) + \sum_{z \in \mathcal{Z}} P(z | b, a) V(b_a^z)
    \end{equation*}
    \Comment{{${V(b_a^z) = \text{heur}(b_a^z)}$ if uninitialized}} \vspace{0.15cm} \label{RTDP-Bel: Q-value}
\State \textbf{Select} action $a_{best}$ that minimizes $Q(b, a)$ \label{RTDP-Bel: update1}
\State \textbf{Update} value $V(b) = Q(b, a_{best})$ \label{RTDP-Bel: update2}
\State \textbf{Sample} $s'$ with probability $\mathcal{T}(s, a_{best}, s')$ \label{RTDP-Bel: succ1}
\State \textbf{Sample} $z$ with probability $\mathcal{O}(s', a_{best}, z)$
\State \textbf{Compute} $b_a^z$ and set $b := b_a^z \text{ and } s:= s'$
\label{RTDP-Bel: succ2}
\EndWhile
\EndWhile
\end{algorithmic}
\end{algorithm}

RTDP-Bel\footnote{We discuss the unapproximated version of RTDP-Bel that maintains the values for each belief state independently.} \cite{rtdp-bel}, outlined in Alg \ref{Alg: RTDP-Bel}, is an adaptation of RTDP \cite{rtdp} to POMDPs. As an asynchronous value iteration algorithm, it converges to the optimal value function and policy for only the relevant regions of the belief space.  The algorithm operates through a series of rollouts (or trials), each beginning from the initial belief state $b_0$ and continuing until a goal belief is reached. At each belief state $b$, all possible actions are evaluated. Evaluating an action $a$ from $b$ involves computing the belief transitions corresponding to the action, i.e., the successor beliefs $b^z_a$. The utility of these actions is then assessed by computing their $Q$-values (Line~\ref{RTDP-Bel: Q-value}). The action with the lowest $Q$-value is selected, and the value estimate of $b$ is updated accordingly (Lines~\ref{RTDP-Bel: update1}–\ref{RTDP-Bel: update2}). The process then proceeds to a successor belief, determined based on the chosen action and a randomly sampled observation (Lines~\ref{RTDP-Bel: succ1}–\ref{RTDP-Bel: succ2}).


A crucial component of the algorithm’s performance is the heuristic function, heur, which initializes the value estimates. A well-informed heuristic accelerates convergence by guiding the search toward the goal quickly. The heuristic must be admissible, i.e., underestimate the optimal value function, for the converged policy to be optimal. Further, \citet{kim2019pomhdp} show that inflating an admissible heuristic by $\epsilon$ allows each search iteration to be more strongly influenced by the heuristic, leading to significantly faster convergence to an $\epsilon$-suboptimal solution.

\subsection{LAO*}


\begin{algorithm}[t]
\caption{LAO*} \label{alg:LAO*}
\begin{algorithmic}[1]
\State $G^* \gets \{b_{0}\}$ 
\While{$G^*$ has non-terminal tip states}
    \State \textbf{Identify} a non-terminal tip state $b$ in $G^*$
    \State \textbf{Evaluate} each action $a \in \mathcal{A}$ from $b$ \label{Lao: Evaluation} 
    \State \textbf{Compute} value of each action as: 
    \begin{equation*}\vspace{-0.15cm}
    Q(b, a) = c(b, a) + \sum_{z \in \mathcal{Z}} P(z | b, a) V(b_a^z)
    \end{equation*}
    \Comment{{${V(b_a^z) = \text{heur}(b_a^z)}$ if uninitialized}} \vspace{0.15cm}    \label{lao: Q-value}
    \State \textbf{Select} action $a_{best}$ that minimizes $Q(b, a)$     
    \State \textbf{Update} value $V(b) = Q(b, a_{best})$ \label{Lao: Update value}
    \State \textbf{Create} set $Z$ with $b$ and ancestors of $b$ in $G^*$
    \State \textsc{ImproveValues}($Z$) \Comment{Value iteration} \label{Lao: Value iteration Z}
    \State \textbf{Update} $G^*$ to current best solution 
    \label{Lao: update G*}
\EndWhile
\State \textsc{ImproveValues}($G^*$) \Comment{Value iteration}
\label{Lao: value iteration G*}
\If{$G^* \neq $ current best solution}
\State \textbf{Update} $G^*$ to current best solution 
\State \textbf{Goto} Line 2
\EndIf
\State \Return $G^*$ \label{Lao: end of function}
\end{algorithmic}
\end{algorithm}

LAO* \cite{lao} outlined in Alg \ref{alg:LAO*}, interleaves forward expansion with dynamic backup updates while ensuring that no states unreachable from the start belief $b_0$ under the current best policy are considered. It maintains a partial solution graph $G^*$, which consists of all belief states reachable from $b_0$ under the current policy. The algorithm iteratively expands and updates $G^*$ until no non-terminal tip states remain. 

A belief state $b$ is classified as a non-terminal tip state if it is neither a goal belief nor has been expanded, meaning its successor beliefs $b^z_a$ have not been generated. When such a state is encountered, all possible actions from $b$ are evaluated, and their corresponding successor beliefs are computed. The optimal action is selected based on the $Q$-value, and the value of $b$ is updated accordingly (Lines~\ref{Lao: Evaluation}–\ref{Lao: Update value}). Next, a set $Z$ is formed, including $b$ and its ancestor beliefs in $G^*$, and value iteration is performed on $Z$ to update both values and optimal actions (Line \ref{Lao: Value iteration Z}). Finally, the solution graph $G^*$ is updated (Line~\ref{Lao: update G*}). This process repeats until no non-terminal tip states remain in $G^*$.


Once no further updates are required, value iteration is performed on all belief states in $G^*$. If the updates lead to changes in the best solution, $G^*$ is revised, and the algorithm returns to Line 2 \footnote{This is a slightly modified interpretation of LAO* to facilitate explanation, though it maintains the same fundamental properties as the original algorithm.}. If the values converge, the algorithm terminates, returning the final solution policy represented by $G^*$ (Lines \ref{Lao: value iteration G*}-\ref{Lao: end of function}). Similar to RTDP-Bel, the heuristic guides the search toward the most promising belief regions, facilitating faster convergence. The same principles of optimality with an admissible heuristic and bounded sub-optimality with an inflated heuristic apply to LAO* as well.

\section{Related Work}



Beyond heuristic search, two widely studied approaches for scalably solving POMDPs are point-based methods and Monte Carlo methods. Point-based POMDP solvers are a class of approximate algorithms designed to efficiently solve large POMDPs by selectively maintaining and updating a representative set of belief points. Solvers such as PBVI \cite{pbvi}, HSVI \cite{hsvi}, and SARSOP \cite{sarsop}, leverage value iteration on a subset of beliefs to compute near-optimal policies while significantly reducing computational complexity. \suhail{However, Point-based backups—a core operation of these solvers—assume the transition and observation models are readily accessible and require querying them at every state in the state space. While feasible for simpler problems, where one can precompute a tabulated transition and observation matrix, our domains involve large state spaces with expensive models, making such enumerations impractical.} Extensive discussions on the many point-based solvers can be found in \citet{shani2013survey}. On the other hand, Monte Carlo methods, particularly suited for online planning, approximate the value of belief states by simulating action sequences and propagating rewards. Works like \citet{pomcp, pomdp_mcts, kurniawati2016online} sample state trajectories and update action values based on rollouts. DESPOT \cite{despot}, a popular approach in this realm, improves efficiency by maintaining a fixed number of sampled scenarios, generating a sparsely sampled belief tree that performs well in problems with large observation spaces. 


In deterministic search, the concept of laziness has been widely explored to defer checking the validity of edges, which is often the computational bottleneck, especially in robotics applications where collision checking is expensive. Lazy Weighted A* (LWA*) \cite{lwa} postpones edge evaluations until expansion, meaning only edges leading to expanded nodes are evaluated. On the other hand, LazySP \cite{lazysp} initially assumes all edges are valid and finds a path to the goal; if an edge is found to be invalid, it is removed from the graph, and the search is rerun, minimizing the number of evaluations. Variants presented in \citet{lrha, bhardwaj2021leveraging, gls} further leverage laziness to accelerate planning. While prior work has focused on deferring edge evaluations in deterministic settings, we extend this concept to stochastic domains by postponing expensive belief transition computations.

\section{Methodology} \label{sec: Methodology}

Many heuristic search methods for POMDPs are on-policy, exploring only belief states in the current best solution policy \cite{rtdp-bel, lao, aems}. They maintain $Q$-values for actions available from each encountered belief state. At a given belief $b$, the best action $a_{best}$, minimizing the $Q$-values is selected, and the search explores only its successor beliefs. The $Q$-value of $a_{best}$ is then updated based on the outcome of this exploration. If this update changes the minimizing action at $b$, $a_{best}$ is updated, and the search proceeds with the new best action and its corresponding successor beliefs. 

The key observation here is that, since only the successor beliefs of the current best action $a_{best}$ are relevant and explored, belief transitions for all other actions remain unused. Consequently, computing these unused transitions introduces unnecessary computational overhead. Hence, we propose to only evaluate the best action $a_{best}$ from any explored belief state. By generating belief transitions only for $a_{best}$, we defer computing the belief transitions for all other actions until the search deems these actions to be promising, i.e., has potential to be part of the solution policy. This ensures that belief transition evaluations are performed only when necessary, significantly reducing computational complexity. Actions that are never considered promising at any point during the search remain unevaluated. This key principle of postponing action evaluations until they are deemed useful forms the essence of lazy search.

However, existing methods typically initialize $Q$-values using a one-step lookahead (Alg \ref{Alg: RTDP-Bel} Line \ref{RTDP-Bel: Q-value}; Alg \ref{alg:LAO*} Line \ref{lao: Q-value}), which involves evaluating each action by computing belief transitions and using successor belief values (initialized with a heuristic) to determine the actions' $Q$-values. This process incurs expensive belief transition computations for all actions, which the lazy search aims to defer. Hence, we propose replacing the one-step lookahead with a fast-to-query estimator that does not require belief transition computations. We denote this estimator as $\hat{Q}_{init}$, as it is used solely for initializing $Q$-values. Under this approach, only the belief transitions for $a_{best}$ are computed. 

The $Q$-estimators our lazy planners leverage are similar to heuristic functions and can be easily constructed for common problems. Just as search algorithms like RTDP-Bel and LAO* converge optimally with an admissible heuristic, $\text{heur}(b) \leq V^*(b)$, their lazy counterparts guarantee optimality if $\hat{Q}_{init}$ is admissible, i.e., $\hat{Q}_{init}(b, a) \leq Q^*(b, a)$.

We first introduce the lazy adaptations of RTDP-Bel and LAO*, followed by a brief discussion of simple techniques for constructing effective estimators.

\subsection{Lazy RTDP-Bel}

\begin{algorithm}
\caption{\textsc{Lazy RTDP-Bel}} \label{Alg: Lazy RTDP-Bel}
\begin{algorithmic}[1]
\Procedure{Lazy RTDP-Bel}{}
\While {\textsc{Not Converged}}
\State $b = b_{0}$ 
\State \textbf{Sample} state $s$ with probability $b(s)$
\While {$b$ is not a goal belief}
\If{$Q$-values not initialized for $b$} 
\label{Lazy RTDP-Bel: Lazy stuff 1}
    \State \textbf{Set} $Q(b,a) = \hat{Q}_{init}(b, a) \,\,\, \forall a \in \mathcal{A}$ 
    \label{Lazy RTDP-Bel: Q estimate}
\EndIf
\Do     \Comment{Do while loop (e.g., to Line \ref{Lazy RTDP-Bel: While})}
\label{Lazy RTDP-Bel: Do}
\State \textbf{Select} $a_{best}$ that minimizes $Q(b, a)$ 
\State \textbf{Evaluate} $a_{best}$ from $b$ if unevaluated
\State \textsc{Update Q-values}($b$)
\doWhile{$a_{best} \neq \argmin_{a \in \mathcal{A}} Q(b,a)$}
\label{Lazy RTDP-Bel: While}
\State \textbf{Update} value $V(b) = Q(b, a_{best})$ 
\State \textbf{Sample} $s'$ with probability $\mathcal{T}(s, a_{best}, s')$ \State \textbf{Sample} $z$ with probability $\mathcal{O}(s', a_{best}, z)$
\State \textbf{Compute} $b_{a_{best}}^z$;  set $b := b_{a_{best}}^z \text{ and } s:= s'$
\EndWhile
\EndWhile
\EndProcedure
\Statex
\Procedure{Update Q-values}{$b$}
\For{$a \in \mathcal{A}$}
\If{$a$ evaluated from $b$}
\State $Q(b, a) = \mathcal{C}(b, a) + \sum_{z \in \mathcal{Z}} P(z | b, a) V(b_{a}^z)$ \\ \vspace{0.1cm}
            \Comment{{${V(b_{a}^z) = \text{heur}(b_{a}^z)}$ if uninitialized}}
\EndIf
\EndFor
\EndProcedure

\end{algorithmic}
\end{algorithm}

The lazy adaptation of RTDP-Bel is outlined in Algorithm~\ref{Alg: Lazy RTDP-Bel} and is fairly straightforward with key differences being from Lines \ref{Lazy RTDP-Bel: Q estimate}-\ref{Lazy RTDP-Bel: While}. When a belief state is encountered for the first time, the $Q$-values of all available actions are initialized using the fast-to-compute estimator, $\hat{Q}_{init}$ (Line~\ref{Lazy RTDP-Bel: Q estimate}). The algorithm then identifies the best action, i.e., the action that minimizes the $Q$-values. Unlike standard RTDP-Bel, only this best action is evaluated, meaning that the belief transitions corresponding to the selected action are computed. The $Q$-value of the evaluated action is then updated accordingly through routine \textsc{Update Q-values}. If this update results in a change in the best action, $a_{best}$ is updated, and the process is repeated until the best action remains unchanged following an update (Lines~\ref{Lazy RTDP-Bel: Do}–\ref{Lazy RTDP-Bel: While}). This approach ensures that the algorithm evaluates only the actions that minimize the $Q$-value for any belief state encountered. By deferring the evaluation of actions until they are deemed promising and potentially part of the (bounded sub-)optimal solution, the method efficiently reduces unnecessary computations. 

\subsection{Lazy LAO*}

\begin{algorithm}[t]
\caption{Lazy LAO*} \label{alg:Lazy LAO*}
\begin{algorithmic}[1]
\Procedure{Lazy LAO*}{}
    \State $G^* \gets \{b_{0}\}$ 
    \While{$G^*$ has non-terminal tip states}
        \State \textbf{Identify} a non-terminal tip state $b$ in $G^*$
        \label{Lazy LAO: Identify non-terminal}
        
        \If{$Q$-values not initialized for $b$}
            \State \textbf{Set} $Q(b,a) = \hat{Q}_{init}(b, a) \,\,\, \forall a \in \mathcal{A}$
            \label{Lazy LAO: Estimate}
        \EndIf
        
        \Do \Comment{Do while loop (e.g., to Line \ref{Lazy LAO: While})}
        \label{Lazy LAO: Do}
        \State \textbf{Select} action $a_{best}$ that minimizes $Q(b, a)$ 
        \State \textbf{Evaluate} $a_{best}$ from $b$ if unevaluated
        \label{Lazy LAO: Evaluate action}
        \State \textsc{Update $Q$-values}(b)
        \doWhile{$a_{best} \neq \argmin_{a \in \mathcal{A}} Q(b,a)$}
        \label{Lazy LAO: While}
        \State \textbf{Update} value $V(b) = Q(b, a_{best})$
        \label{Lazy LAO: Update value}
        \State \textbf{Create} set $Z$ with $b$ and ancestors of $b$ in $G^*$

        \State \textsc{ImproveValues}($Z$)

        \State \textbf{Update} $G^*$ to current best solution
        \label{Lazy LAO: update current best solution 1}
    \EndWhile 
    \label{Lazy LAO: Return to Line 3 1}
    \State \textsc{ImproveValues}($G^*$)
    \If{$G^* \neq $ current best solution}
    \State \textbf{Update} $G^*$ to current best solution 
    \label{Lazy LAO: update current best solution 2}
    \State \textbf{Goto} Line 3 
    \label{Lazy LAO: Return to Line 3 2}
    \EndIf
    \State \Return $G^*$
    \label{Lazy LAO: End of function}
\EndProcedure

\Statex
\Procedure{ImproveValues}{$\mathcal{B}$}
\While{Error bound greater than $\epsilon$}
\For{$b \in \mathcal{B}$}
\State \textsc{Update Q-values}($b$)
\State \textbf{Select} action $a_{best}$ that minimizes $Q(b, a)$ 
\If {$a_{best}$ not evaluated from $b$}
\label{Lazy LAO: Improve value return}
\State return
\EndIf
\State Update value $V(b) = Q(b, a_{best})$ 
\EndFor
\EndWhile
\EndProcedure

\end{algorithmic}
\end{algorithm}

The lazy adaptation of LAO* outlined in Alg \ref{alg:Lazy LAO*} follows the same fundamental structure as vanilla LAO*. Like its standard counterpart, the algorithm maintains the current best partial solution graph, $G^*$, and iteratively identifies non-terminal tip states for expansion. However, we slightly modify the definition of a non-terminal tip state. In vanilla LAO*, a non-terminal tip state is a non-goal belief whose actions have not been evaluated. In our adaptation, a non-terminal tip state is defined as either (i) a non-goal belief with no evaluated actions or (ii) a non-goal belief whose best action—the action with the lowest $Q$-value—has not yet been evaluated. 

Using this modified definition, we identify a non-terminal tip state in $G^*$ (Line \ref{Lazy LAO: Identify non-terminal}). If the belief state has not been encountered before, its $Q$-values are initialized using the quick-to-compute estimator $\hat{Q}_{init}$ (Line \ref{Lazy LAO: Estimate}). As in Lazy RTDP-Bel, the algorithm identifies the action minimizing the $Q$-values and evaluates only this action. The evaluated action’s value is then updated accurately. If this update changes the best action, the newly identified best action is evaluated, and the process repeats (Lines \ref{Lazy LAO: Do}-\ref{Lazy LAO: While}). This ensures only promising actions from a state are evaluated. After this step, the algorithm proceeds like LAO*, using the $Q$-value of the best action to update $b$. The set $Z$, consisting of $b$ and its ancestors within $G^*$, is then constructed, and value improvements are performed. This continues until no non-terminal tip states remain in $G^*$. At this point, the values of all states in $G^*$ are further refined until either $G^*$ is modified or the values converge (Lines \ref{Lazy LAO: Update value}-\ref{Lazy LAO: End of function}).

A key difference from vanilla LAO* lies in how values are improved in the \textsc{ImproveValues} routine. In standard LAO*, this routine is simple value iteration, assuming all actions from a belief state are either evaluated or can be cheaply evaluated. In contrast, our adaptation accounts for lazy evaluations. If, during backup, the best action of a state changes to an unevaluated one, backups terminate (Line \ref{Lazy LAO: Improve value return}). If the best action for any state in either set $Z$ or $G^*$ changes, it implies that the current best partial solution has been modified, necessitating an update to $G^*$ (Lines \ref{Lazy LAO: update current best solution 1}, \ref{Lazy LAO: update current best solution 2}), followed by which the algorithm returns to Line 3 (Lines \ref{Lazy LAO: Return to Line 3 1}, \ref{Lazy LAO: Return to Line 3 2}). Notably, actions are evaluated only in Line \ref{Lazy LAO: Evaluate action}; at no other point are they explicitly evaluated. These modifications ensure that action evaluations are deferred as much as possible, providing sizeable computational savings. 

\begin{theorem}
If the estimator $\hat{Q}_{init}$ is an underestimate, i.e.,  $\hat{Q}_{init}(b, a) \leq Q^*(b, a) \,\, \forall \,\, b \in \mathcal{B}, a \in \mathcal{A}$, where $Q^*(b, a)$ corresponds to the optimal $Q$-value, then both Lazy RTDP-Bel and Lazy LAO* will converge to the optimal policy.
\end{theorem}

\begin{proof}
The idea is similar to the optimality proof for vanilla RTDP-Bel and LAO*. Since $Q_{init}$ is an underestimate, Bellman backups ensure $Q(b,a) \leq Q^*(b,a)$ for all belief-action pairs encountered throughout the search. Each backup progressively increases the value of the current best action (the action that minimizes $Q$-values) at belief states along the current policy, bringing it closer to the optimal value. If a suboptimal action initially minimizes the $Q$-value at a belief state, updates eventually raise its value above that of the optimal action, ensuring the correct action is selected and the search converges to the optimal policy.
\end{proof}

\subsubsection{Full horizon laziness} \label{subsection: FH Lazy}

\suhail{
 While the bottleneck in computing belief transitions is typically expensive transition or observation models, in some planning problems (as shown in Section \ref{sec: results}), the bottleneck is verifying an action's validity from a belief state. Here, successor beliefs can be efficiently computed, but determining action feasibility is computationally expensive. For example, in navigation tasks, an action may be deemed invalid from a belief $b$ if executing it from any state $s$ with $b(s) >0$ could cause a collision. Verifying feasibility requires expensive collision checks, while computing successor beliefs assuming the action is valid is relatively cheap (as motion-primitives are known). In such cases, we can adopt what we call full-horizon lazy planning. Here, the policy is initially computed (using your favorite solver) assuming all actions are valid. Only actions in the computed policy undergo the expensive validation process. Invalid actions are marked infeasible, and the planner is queried again. This repeats until the policy consists entirely of valid actions. Similar to LazySP in deterministic settings, this approach defers the bottleneck operation until a complete policy is computed, minimizing the number of expensive queries. However, it is only applicable when successor beliefs can be computed efficiently and the primary computational cost lies in verifying action feasibility.
 }


\subsection{Q-value estimators for lazy planning}
Since the estimator $\hat{Q}_{init}$ is only used to initialize the $Q$-values, it can be interpreted as an approximation of the initial $Q$-values, $Q_{init}(b,a)$, defined as

\vspace{-0.2cm}
\begin{equation} \label{Eqn: Q_{init}} \vspace{-0.2cm}
    Q_{init}(b, a) = \mathcal{C}(b, a) + \sum_{z \in Z} P(z| b, a) \cdot \text{heur}(b^z_a).
\end{equation}

A strong correlation between $\hat{Q}_{init}$ and $Q_{init}$ enhances the effectiveness of lazy planners. To guarantee optimality, $\hat{Q}_{init}$ should underestimate $Q_{init}$ (assuming the heuristic is admissible). However, a close approximation suffices for strong performance. Prior work in deterministic planning has explored using learned models to predict heuristic values \cite{veerapaneni2023learning, takahashi2019learning}. We see potential in extending these ideas to stochastic settings. Domain-specific knowledge can also be leveraged to construct useful estimators. 

Our focus is not on designing sophisticated estimators but on demonstrating the utility of lazy planners when paired with an estimator. In addition to the approaches mentioned, we present a simple domain-independent subsampling method that proved effective in our empirical analysis.

The key computational bottleneck in computing \(Q_{init}(b, a)\) (Eq.~\ref{Eqn: Q_{init}}) is determining successor beliefs \(b^z_a\). This requires evaluating action \(a\) on belief \(b\), which is expensive as \(b\) is a distribution over the state space \(\mathcal{S}\). The action must be evaluated for every state \(s\) where \(b(s) \neq 0\).\footnote{In many robotics problems, the state space is large. While the support of a typical belief state during search is substantial, it still represents only a fraction of the total state space.} Given the number of belief states typically encountered during search, the computational cost is significant.


To address this, we estimate \(Q_{init}(b,a)\) using subsampling. We construct a discrete approximation \(\hat{b}\) of \(b\) by sampling states \(s \sim b\) until \(\text{supp}(\hat{b})\) reaches a predefined threshold \(\Delta\). The frequency of occurrence of each sampled state determines its probability in \(\hat{b}\). Since \(Q_{init}(\hat{b}, a)\) can be computed efficiently due to the reduced support and approximates \(Q_{init}(b, a)\) well, we define the subsampled estimator as $\hat{Q}_{init}(b, a) = Q_{init}(\hat{b}, a)$.

The subsampling estimator in its vanilla form is effective when the heuristic satisfies 
\(\text{heur}(b) = \mathbb{E}_{s \sim b} \text{heur}(s)\), a property common in robotics planning, especially in goal-directed tasks (e.g., an autonomous vehicle reaching a goal). However, if the heuristic does not satisfy this property——such as when it measures belief entropy (e.g., in active localization, where the cost of reducing uncertainty correlates well with the current uncertainty level \cite{spaan2008cooperative})—correction terms can be introduced to mitigate bias. \suhail{Although effective, the subsampled estimates are not guaranteed to be conservative.} A detailed discussion of the subsampling estimators for our evaluated domains is included in the Appendix.


Additionally, in problem domains where belief transition costs stem primarily from the observation model, while querying the transition model is inexpensive, we propose an estimator inspired by \(Q^{MDP}\) heuristics \cite{qmdp2}. Commonly used as admissible estimates in POMDP planning, \(Q^{MDP}\) heuristics assume the state becomes fully observable after a single action. We define this estimator as:

\vspace{-0.32cm}
\begin{equation*} \vspace{0.02cm} 
    \hat{Q}_{init}^{MDP}(b, a) = \sum_{s\in \mathcal{S}} b(s) \biggl(c(s,a) + \sum_{s'} \mathcal{T}(s, a, s') \text{heur}(s')\biggr)
\end{equation*}

The estimator \(\hat{Q}_{init}^{MDP}\) is conservative if \(\text{heur}(s')\) underestimates the cost of reaching the goal set \(G\) from \(s'\) in the underlying MDP defined by the transition model $\mathcal{T}$.

\subsection{Discussion}

\suhail{The idea of lazily evaluating belief transitions by deferring their computation until necessary extends well beyond the two specific instantiations discussed above. It naturally applies to variants such as ILAO*~\cite{lao}, IBLAO*~\cite{iblao}, and AO*~\cite{ao}. POMHDP~\cite{kim2019pomhdp}, which augments RTDP-Bel with multiple heuristics, can also incorporate laziness by using separate $Q$-estimators for each heuristic. Beyond these, the concept extends to other heuristic search methods like AEMS2~\cite{aems}, BI-POMDP~\cite{bipomdp}, and ~\citet{satialave}. With a few modifications, we also see potential in extending it to off-policy algorithms like AEMS~\cite{aems}, which expand belief nodes outside the current best policy. Overall, lazy heuristic search for POMDP planning forms a broad class of algorithms. Finally, the core idea of deferring transition computations via $Q$-value estimation can be applied to MDPs as well.}

\section{Results} \label{sec: results}
The performance of lazy planners is evaluated across three problem domains: manipulation for pose estimation, indoor navigation, and outdoor rough terrain navigation. Without loss of generality, we represent the belief distribution as a discrete particle set and assume perfect observations.\footnote{This is not a limitation; our framework can directly support other belief representations and noisy observation models.}  For all problem domains, the subsampling estimator \(\hat{Q}_{init}(b, a) = Q_{init}(\hat{b}, a)\) samples \(\hat{b}\) to contain 15\% of the total particles in \(b\), i.e., \(\Delta\) is set to 15\% of the support of \(b\).
The reported planning time and cost statistics are averaged over runs where all algorithms solving at least 20\% of problems have succeeded. Statistics for algorithms solving fewer than 20\% are excluded to maintain representativeness.

\subsection{Manipulation for pose estimation using contacts}
\begin{table}[h]
    \centering
    \small 
    \renewcommand{\arraystretch}{1.2} 
    \setlength{\tabcolsep}{5pt} 
    \begin{tabular}{lccc}
        \toprule
        \textbf{Planner} & \textbf{Success Rate} & \textbf{Planning Time} & \textbf{Cost} \\
        \midrule
        RTDP-Bel       & 0.87  & 307.49  & 0.69 \\
        Lazy RTDP-Bel  & 1.0   & 183.18  & 0.71 \\
        LAO*           & 0.93  & 253.67  & 0.69 \\
        Lazy LAO*      & 1.0   & 136.73  & 0.70 \\
        \bottomrule
    \end{tabular}
    \caption{Performance comparison on the manipulation for pose estimation problem.}
    \label{tab:manipulation_planner_comparison}
    \vspace{-1em}
\end{table}

\begin{figure}
\centering
  \includegraphics[width=0.45\textwidth]{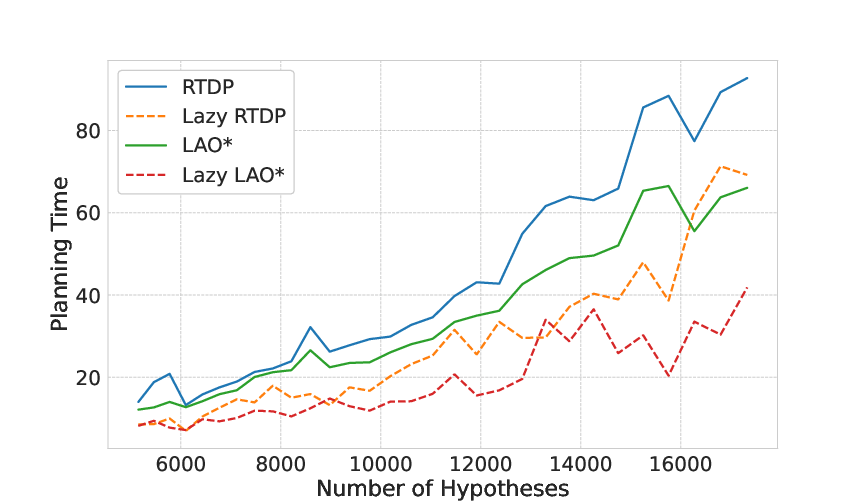} 
  \caption{Planning times as a function of object pose uncertainty on the manipulation for pose estimation problem.} \label{fig:num_hypothesis_vs_time} \vspace{-0.3cm}
\end{figure} 

\begin{figure*}
\centering
  \includegraphics[width=0.95\textwidth]{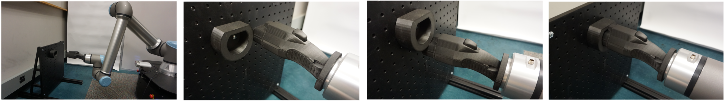}
  \caption{
  The robot utilizing contact observations from different actions to reduce uncertainty about the pose of the object of interest, i.e., the charger port, before completing the insertion task.
}\label{fig:manipulation_real_robot}\vspace{-0.3cm}
\end{figure*}

In this domain, the goal is to localize a target object $T$ (a charger plug) and complete an insertion task. The plug's pose is not directly observable; instead, the robot manipulator infers it through binary contact observations (Fig. \ref{fig:manipulation_real_robot}). This is useful when visual feedback is ineffective due to occlusions or poor lighting. The presence or absence of contact during actions reduces pose uncertainty. For instance, if no contact is observed while sweeping a volume, the object is absent in that region. Conversely, contact confirms the object's presence in the swept volume. 


Let the particle set \( H = \{h_1, h_2, \dots\} \) represent the distribution of hypothesis poses for the target object $T$. The action space includes unit robot movements and compliant controllers that track object surfaces upon contact. The observation space consists of the robot state $q$ (from encoder readings) and a binary collision signal (from an F/T sensor). The belief state is given by \( b = \{q, H\} \). The objective is to find a policy that fully localizes the target object ($|H|=1$) while minimizing the robot's expected travel distance. For a given object pose $h_i$ and action $a$, computing the expected observation—i.e., the sequence of contact signals and robot states—requires forward simulating the action with the object model at $h_i$ which involves complex mesh-to-mesh collision checks. Computing successor beliefs for \( b = \{q, H\} \) requires simulating actions for all $h \in H$, making the process computationally expensive. Further details about the problem formulation are in \citet{saleem2023preprocessing, saleem2024pomdp}. Since this is an active information-gathering problem, we define the belief state heuristic as a scaled version of the size of the set $H$ (a measure of the belief entropy). We use the subsampling estimator with a minor modification (to ensure unbiased estimation), details about which can be found in the Appendix.

Results averaged over 100 runs are presented in Table \ref{tab:manipulation_planner_comparison}. In each run, the planner resolves 3D positional uncertainties in the pose of \(T\), with uncertainty along each axis randomly chosen between 4mm and 80mm. This uncertainty is discretized at a 2mm resolution to form the hypothesis set \(H\), meaning a 30mm uncertainty per axis results in a start belief with \(15^3\) hypotheses. Each planner was given a 500-second timeout. The results show lazy planners outperform their vanilla counterparts by solving problems more efficiently through fewer belief transition evaluations. Lazy planners incurred roughly 50-60\% of the planning times of their vanilla counterparts, enabling them to tackle harder problems and achieve higher success rates. Figure \ref{fig:num_hypothesis_vs_time} shows planning times as a function of the number of hypothesis poses (\(|H|\)) in the start belief, i.e., start state uncertainty. The computational cost of computing transitions for a belief state is proportional to the number of hypotheses in it. As the start state uncertainty increases, the number of belief states with a high number of hypothesis poses encountered during the search also increases. As a result, the speedups offered by the lazy planners grow with the start state uncertainty.


\subsection{Indoor navigation}

\begin{figure}
\centering
  \includegraphics[width=0.9\linewidth]{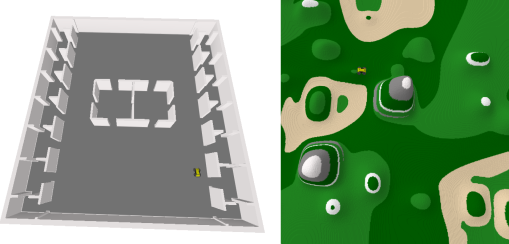}
  \caption{Visualization of the indoor (left) and outdoor navigation environments (right).}\label{fig: indoor_outdoor_map}\vspace{-0.3cm}
\end{figure}

This is a 3D navigation domain, where the robot state is defined by the tuple $(x, y, \theta)$. The map is known but the state of the robot is not directly observable, instead, the robot relies on a 1D LiDAR sensor mounted on it for perception. Computing expected observations from a given state require expensive raycasting operations to simulate the sensor. The belief over the robot's state is represented as a set of particles $H = \{h_1, h_2 ... \}$, each corresponding to a hypothesis robot pose $h_i$. Consequently, computing expected observations from a belief state requires performing raycasting from all hypothesis poses in the belief, making the belief transition computations expensive. The discrete action set $\mathcal{A}$ consists of predefined motion primitives. We evaluate two variants of the problem. For both variants we use 3 indoor environments similar to the one we see in Fig \ref{fig: indoor_outdoor_map}.

\subsubsection{With stochastic transition dynamics}
\begin{table}[h]
    \centering
    \small 
    \renewcommand{\arraystretch}{1.2} 
    \setlength{\tabcolsep}{5pt} 
    \begin{tabular}{lccc}
        \toprule
        \textbf{Planner} & \textbf{Success Rate} & \textbf{Plan Time} & \textbf{Cost} \\
        \midrule
        RTDP-Bel                                              & 0.81	& 153.04   & 14.14 \\
        Lazy RTDP-Bel \hspace{0.1cm}($\hat{Q}_{init}$)        & 0.90	& 32.46	   & 14.26 \\
        Lazy RTDP-Bel \hspace{0.1cm}($\hat{Q}_{init}^{MDP}$)  & 0.90	& 38.76    & 14.14 \\
        LAO*                                                  & 0.71	& 185.37   & 14.14 \\
        Lazy LAO*  \hspace{0.1cm}($\hat{Q}_{init}$)           & 0.77	& 33.54	   & 14.24  \\
        Lazy LAO*  \hspace{0.1cm}($\hat{Q}_{init}^{MDP}$)     & 0.77	& 40.43	   & 14.14 \\
        \bottomrule
    \end{tabular}
    \caption{Performance comparison on the indoor navigation with stochastic transition dynamics problem.}
    \label{tab:indoor_transition_planner_comparison}
\end{table}

In this variant, uncertainty in the robot's state arises from stochastic transitions. Certain regions in the map induce probabilistic motions, where the robot reaches the intended end state with probability 0.5, while slipping left or right with equal likelihood. The robot's start state is known, and the objective is to reach a goal region $G$. The belief state is represented as a set of weighted particles, with the weight corresponding to the probability of each robot pose. Heuristic is the expected distance of the states in the belief from $G$, $heur(b) = \mathbb{E}_{s \sim b} \,\, \text{Dist}(s, G)$. \text{Dist} is the minimum distance of $s$ from $G$ on the MDP defined by the transition model $\mathcal{T}$, making the heuristic admissible. Alternatively, \text{Dist} can be the Euclidean distance between $s$ and $G$, providing a quick but inadmissible heuristic. We evaluate with the subsampled estimator, and since the computational expense of the belief transitions stems from the observation model, we also evaluate with the $\hat{Q}_{init}^{MDP}$ estimator. Results presented in Table \ref{tab:indoor_transition_planner_comparison} are averaged over 200 runs, with random start states and goal regions. The timeout was 300 seconds. We observed that the lazy planners were particularly effective in this problem, improving planning speeds by 5x while computing policies with a marginally higher expected cost. The $\hat{Q}_{init}^{MDP}$ estimator resulted in slightly higher planning times in comparison to the subsampling estimator, however, it is guaranteed to be a conservative estimate \suhail{(while the subsampled estimator is not)}, resulting in solution costs identical to that of the vanilla planners.

\subsubsection{With start state uncertainty}
\begin{table}[h]
    \centering
    \small 
    \renewcommand{\arraystretch}{1.2} 
    \setlength{\tabcolsep}{5pt} 
    \begin{tabular}{lccc}
        \toprule
        \textbf{Planner} & \textbf{Success Rate} & \textbf{Planning Time} & \textbf{Cost} \\
        \midrule
        \multicolumn{4}{c}{\textbf{Goal Directed Navigation}} \\
        \midrule
        RTDP-Bel      & 0.73         & 139.32        & 19.57 \\
        Lazy RTDP-Bel & 0.89         & 84.91         & 19.61 \\
        LAO*          & 0.87         & 96.47         & 19.57 \\
        Lazy LAO*     & 0.99         & 53.42         & 19.60 \\
        \midrule
        \multicolumn{4}{c}{\textbf{Information Gathering}} \\
        \midrule
        RTDP-Bel       & 1.0       & 39.17  & 14.39 \\
        Lazy RTDP-Bel  & 1.0       & 21.32  & 14.25 \\
        LAO*           & 1.0       & 32.21  & 14.42 \\
        Lazy LAO*      & 1.0       & 16.75  & 14.87 \\
        \bottomrule
    \end{tabular}
    \caption{Performance comparison on the indoor navigation with start state uncertainty problem.} \vspace{-0.3cm}
    \label{tab:indoor_start_state_uncertainty_planner_comparison}
\end{table}


In this variant, uncertainty in the robot's state arises from start state uncertainty, defined by an initial set of equally likely hypothesis poses. The transition dynamics is deterministic, allowing the belief to be represented as a set of unweighted particles. This problem has two versions: (1) the robot must localize itself, and (2) the robot must reach a goal region \(G\). For the information-gathering version, the goal belief, belief heuristic, and \(Q\)-estimator are defined identically to the manipulation problem. In the goal-directed version, they follow the definitions used for stochastic transition dynamics. Results for both versions, averaged over 200 runs, are presented in Table~\ref{tab:indoor_start_state_uncertainty_planner_comparison}, with each run using a random set of 100 to 250 start hypothesis poses. Similar to the manipulation domain, the lazy planners on average took 60\% of the planning time of their vanilla counterparts at the expense of marginally higher costs.

\subsection{Rough terrain navigation}
\begin{table}[h]
    \centering
    \small 
    \renewcommand{\arraystretch}{1.2} 
    \setlength{\tabcolsep}{5pt} 
    \begin{tabular}{lccc}
        \toprule
        \textbf{Planner} & \textbf{Success Rate} & \textbf{Planning Time} & \textbf{Cost} \\
        \midrule
        \multicolumn{4}{c}{\textbf{Goal Directed Navigation}} \\
        \midrule
        RTDP-Bel         & 0.01         & -             & -     \\
        Lazy RTDP-Bel    & 0.21         & 1297.33       & 12.66 \\
        FH Lazy RTDP-Bel & 0.90         & 109.9         & 12.66 \\
        LAO*             & 0.09         & -             & -     \\
        Lazy LAO*        & 0.33         & 886.64        & 12.66 \\
        FH Lazy LAO*     & 0.90         & 139.52        & 12.66 \\
        \midrule
        \multicolumn{4}{c}{\textbf{Information Gathering}} \\
        \midrule
        RTDP-Bel         & 0.07         & -             & -     \\
        Lazy RTDP-Bel    & 0.19         & 1015.51       & 14.58 \\
        FH Lazy RTDP-Bel & 0.83         & 110.38        & 14.58 \\
        LAO*             & 0.09         & -             & -     \\
        Lazy LAO*        & 0.26         & 642.78        & 14.58 \\
        FH Lazy LAO*     & 0.81         & 163.79        & 14.58 \\
        \bottomrule
    \end{tabular}
    \caption{Comparison on the outdoor navigation problem.}
    \label{tab:outdoor_navigation_planner_comparison} \vspace{-0.3cm}
\end{table} 

This is a 3D outdoor navigation problem with deterministic transitions and observations. Uncertainty in the robot’s state arises from start state uncertainty, with the belief represented as a set of unweighted particles. Unlike indoor navigation, the robot lacks a 1D LiDAR sensor; instead, the environment contains landmarks that the robot can observe if within a predefined distance, making the observation model efficient to query. Consequently, both the transition and observation models are computationally cheap. However, in this domain, an action is feasible from a belief state only if it can be accurately tracked from all hypothesis states in the belief. Due to rough terrain, certain states may cause the robot to get stuck in pits, fail to climb inclines, or execute motions inaccurately. Hence, validating an action requires expensive rollouts in a physics simulator, making action verification the primary bottleneck. Since the expected trajectory that will be tracked from the hypothesis states if the action were valid is known, and the observation model is efficient, belief transitions can be computed quickly. The bottleneck operation is simply the verification of the validity of actions.

Similar to the indoor navigation case, we evaluate two versions of the problem. Table \ref{tab:outdoor_navigation_planner_comparison} presents the performance averaged over 100 runs, with each run containing 30–50 randomly sampled start state hypothesis and a timeout of 2400 seconds. FH Lazy RTDP-Bel and FH Lazy LAO* correspond to the full-horizon lazy variants discussed in Section \ref{subsection: FH Lazy}. We observe that the full-horizon lazy planners demonstrate strong performance boosting success rates by close to 90\% in comparison to the vanilla counterparts. By deferring action evaluations until after policy computation, the full-horizon lazy planners achieve a 10x speed up over the 1-step lazy planners. The high computational expense of the physics simulations make the impact of laziness pronounced in this domain, highlighting the utility of lazy planners.

\section{Conclusion}
In this work, we introduced Lazy RTDP-Bel and Lazy LAO*, two specific instances of a broader class of lazy heuristic search solvers for POMDPs with expensive-to-compute belief state transitions. These algorithms leverage $Q$-value estimates to defer expensive belief transition evaluations until they are truly necessary. By postponing these computations, our approach significantly enhances planning efficiency while preserving solution quality. We validated our methods across three challenging robotics problems, achieving substantial speedups over conventional solvers in all of them.

\section*{Acknowledgements}
This work was supported in part by Honda Research Institute and in part by the ARO-sponsored DURIP grant W911NF-21-1-0050.

\section*{Appendix}
Here, we provide details about the $Q$-estimators, $\hat{Q}_{init}$, used for the problem domains we evaluated our framework on. We also discuss and evaluate a probabilistically conservative estimator for the manipulation using contacts task to highlight the scope of estimators that can be developed for robotics problems. 

\subsection{Manipulation for pose estimation using contacts}
In this domain, the goal is to localize a target object $T$ which in this case is a charger plug, and complete a plug insertion task. However, the pose of the plug is not directly observable, instead, a robot manipulator needs to use binary contact observations to infer the pose. A problem particularly relevant in scenarios where visual feedback is hindered by occlusions or poor lighting. The presence or absence of contact during actions helps reduce uncertainty about the object's pose. For example, if no contact is observed while sweeping a volume, the object is guaranteed not to be in that region. Conversely, if contact is made, the object must be located such that it occupies some portion of the swept volume.

Let the particle set \( H = \{h_1, h_2, \dots\} \) represent the distribution of hypothesis poses for the target object \(T\). The action space includes unit robot movements and compliant controllers that track object surfaces upon contact. The observation space consists of the robot state $q$ (from encoder readings) and a binary collision signal (from an F/T sensor). The belief state is given by \( b = \{q, H\} \). The objective is to find a policy that fully localizes the target object (belief state with \(|H| = 1\)) while minimizing the robot's expected distance travelled. 

Given a target object pose $h_i$, the expected observation—i.e., the sequence of contact signals and robot states—must be computed by forward simulating the action with the object model at pose $h_i$. This involves complex mesh-to-mesh collision checks to detect exact contact points and compute the trajectory tracked by the compliant controller. To compute the successor beliefs for a given belief state, the action must be forward simulated for all hypothesis poses $h_i \in H$. This makes evaluating an action and computing successor beliefs (and their probabilities) computationally expensive, especially when the belief state contains thousands of hypothesis poses. 


Given a belief state \( b = \{r, H\} \), an action \( a \), and an observation \( z \), the corresponding successor belief \( b_a^z = \{r', H_a^z\} \) consists of all hypothesis poses under which \( z \) would have been observed when executing \( a \) from \( r \). Since the particles in the hypothesis set are unweighted, the probability of observing \( z \) is given by  $P(z| b, a) = \frac{|H_a^z|}{|H|}$. Since this is an active localization problem, we employ a scaled version of the entropy of a belief state, measured by the size of the hypothesis set \( |H| \), as our heuristic function. Specifically, we define the heuristic as $\text{heur}(b^z_a) = \alpha |H^z_a|$, where \( \alpha \) is a scaling factor. This formulation reflects the intuition that a larger number of particles in the current belief state corresponds to increased localization effort. Consequently, we can rewrite \( Q_{init} \) as follows:

\begin{equation}
    \begin{split}
        Q_{init}(b, a) &= \mathcal{C}(b, a) + \sum_{z \in \mathcal{Z}} P(z | b, a) \, \text{heur}(b_{a}^z) \\
        \implies Q_{init}(b, a) &= \mathcal{C}(b, a) + \sum_{z \in \mathcal{Z}} \frac{|H^z_a|}{|H|} (\alpha |H^z_a|) \\
        \implies Q_{init}(b, a) &= \mathcal{C}(b, a) + \alpha \sum_{z \in \mathcal{Z}} \frac{|H^z_a|^2}{|H|}
    \end{split}
\end{equation}

Hence, if we can compute \( |H^z_a| \), then we can compute \( Q_{init}(b, a) \). However, computing \( |H_a^z| \) would involve rolling out the action on all particles in \( H \).

Now, the objective is to subsample a belief state \( \hat{b} = \{r, \hat{H}\} \) from \( b \), evaluate the action on the subsampled belief state, compute \( \hat{H}^z_a \), and utilize it to estimate \( H^z_a \). For notational simplicity, let \( n \) represent the number of particles in \( H \). The subset \( \hat{H} \) is created by sampling \( k \) particles from \( H \), where \( k \ll n \). Hence, evaluating the action on \( \hat{b} \) and computing \( \hat{H}^z_a \) is significantly faster.

If we assume that the observations obtained when executing \( a \) on different particles are independent (a simplifying assumption commonly employed), we can view the problem as follows: each particle in \( H \) contains a hidden variable representing its assignment to one of \( |\mathcal{Z}| \) groups, where the assignment represents the observation obtained when executing \( a \) on the particle. The goal is to predict the distribution, i.e., the number of particles in each group, when the assignments of only a subset of the particles are visible. Based on this perspective, \( \frac{n}{k} \hat{H}^z_a \) serves as an unbiased estimator of \( H^z_a \). Since the particles are randomly chosen to be part of \( \hat{H} \), the probability of a particle being in \( \hat{H} \) is given by \( \frac{k}{n} \), and since the observations of the particles are independent, we can write:

\begin{equation}
\begin{split}
    \mathbb{E}_{\hat{H} \sim H} \, (|\hat{H}^z_a|) &= \frac{k}{n} |H^z_a| \\
    \implies \mathbb{E}_{\hat{H} \sim H} \left(\frac{n}{k}|\hat{H}^z_a|\right) &= |H^z_a|
\end{split}    
\end{equation}

This implies that we can write the unbiased subsampled estimator for the manipulation problem as:

\begin{equation} \label{Eqn: Manipulation Q_hat}
    \hat{Q}_{init}(b, a) = \mathcal{C}(b, a) + \alpha \left(\frac{n}{k}\right)^2 \sum_{z \in \mathcal{Z}} \frac{|\hat{H}^z_a|^2}{|H|}
\end{equation}

As we can see, this is simply a minor manipulation of \( Q_{init}(\hat{b}, a) \), which can be written as:

\begin{equation}
    Q_{init}(\hat{b}, a) = \mathcal{C}(b, a) + \alpha \sum_{z \in \mathcal{Z}} \frac{|\hat{H}^z_a|^2}{|H|}
\end{equation}

The estimator \( \hat{Q}_{init}(b, a) \) essentially contains a correction term in the form of the squared sampling ratio \( \left(\frac{n}{k}\right)^2 \), to make the estimate unbiased and effective for our problem domain.

\subsubsection{Probabilistically Conservative Estimator}
While the subsampled estimator defined in Eq.~\ref{Eqn: Manipulation Q_hat} is an effective and unbiased approximation, it does not guarantee conservativeness, i.e., \( \hat{Q}_{init}(b, a) \) is not necessarily less than or equal to \( Q_{init}(b, a) \). Since the heuristic in this problem is based on entropy which is not an admissible heuristic, there might not be a need for a conservative estimator. Nonetheless, we experimented with a modified version of the subsampled estimator that makes it probabilistically conservative, i.e., with 95\% confidence $\hat{Q}_{init}(b, a) \leq Q_{init}(b, a)$. 

To achieve this, we leverage Hoeffding's inequality. Given independent Bernoulli random variables \( X_i \), let \( S_n = \sum X_i \). Hoeffding’s inequality states:

\begin{equation}
    P(S_n - \mathbb{E}[S_n] \geq \epsilon) \leq \exp \left(\frac{-2\epsilon^2}{n}\right).
\end{equation}

If we assume the membership of the particles to each of the different observations to be independent (as we have previously), we can write,

\begin{equation}
    P\left(|\hat{H}^z_a| - |H^z_a| \frac{k}{n} \geq \epsilon \right) \leq \exp\left(\frac{-2\epsilon^2}{k}\right).
\end{equation}

Setting \( \epsilon = 1.22 \sqrt{k} \), we obtain:

\begin{equation}
    P\left(\left(|\hat{H}^z_a| - 1.22 \sqrt{k}\right)\frac{n}{k} \geq |H^z_a|\right) \leq 0.05.
\end{equation}

This result implies that with 95\% confidence, the term \( (|\hat{H}^z_a| - 1.22 \sqrt{k})\frac{n}{k} \) provides a conservative estimate of \( H^z_a \). Consequently, we define the probabilistically conservative subsampling estimator as:

\begin{equation} \label{Eqn: Conservative Q}
    \hat{Q}_{init}(b, a) = \mathcal{C}(b, a) + \alpha \sum_{z \in \mathcal{Z}} \frac{\big((|\hat{H}^z_a| - 1.22 \sqrt{k})\frac{n}{k}\big)^2}{|H|}.
\end{equation}

Thus, Eq.~\ref{Eqn: Conservative Q} defines an estimator that is conservative with 95\% confidence. We believe there is significant potential to explore similar and innovative estimators for common robotics problems. This example was simply to highlight the scope. Additionally, we also evaluated the conservative estimator on the manipulation using contacts task, the results for which have been included in Table \ref{tab:manipulation_planner_comparison_pce}. We represent this probabilistically conservative estimator as $\hat{Q}_{init}^{pce}$, while the vanilla subsampling estimator is continued to be represented with $\hat{Q}_{init}$.

\begin{table}[h]
    \centering
    \small 
    \renewcommand{\arraystretch}{1.2} 
    \setlength{\tabcolsep}{5pt} 
    \begin{tabular}{lccc}
        \toprule
        \textbf{Planner} & \textbf{Success Rate} & \textbf{Planning Time} & \textbf{Cost} \\
        \midrule
        RTDP-Bel                                & 0.87  & 307.49  & 0.69 \\
        Lazy RTDP-Bel ($\hat{Q}_{init}$)        & 1.0   & 183.18  & 0.71 \\
        Lazy RTDP-Bel ($\hat{Q}_{init}^{pce}$)  & 1.0   & 221.20  & 0.71 \\
        LAO*                                    & 0.93  & 253.67  & 0.69 \\
        Lazy LAO*  ($\hat{Q}_{init}$)           & 1.0   & 136.73  & 0.70 \\
        Lazy LAO*  ($\hat{Q}_{init}^{pce}$)     & 1.0   & 175.14  & 0.72 \\
        \bottomrule
    \end{tabular}
    \caption{Performance comparison on the manipulation for pose estimation problem.}
    \label{tab:manipulation_planner_comparison_pce}
\end{table}

\subsection{Navigation}
For the information-gathering version of the navigation problem, we defined the heuristic, identical to the manipulation problem to be the scaled version of the number of particles in the current belief. Hence, the subsampled estimator $\hat{Q}_{init}$ was defined identically as in Eqn. \ref{Eqn: Manipulation Q_hat} (which includes the additional bias correction term).

For the goal-directed version of the problem, the subsampled estimator was defined as outlined in Section 4 of the paper (without any additional bias correction terms).
\begin{equation}
    \hat{Q}_{init}(b, a) = Q_{init}(\hat{b}, a)
\end{equation}

The $Q_{MDP}$ estimator, $\hat{Q}_{init}^{MDP}$ was also utilized as defined in Section 4 of the paper, 

\begin{equation}
    \hat{Q}_{init}^{MDP}(b, a) = \sum b(s) \biggl(c(s,a) + \sum_{s'} T(s, a, s') \text{heur}(s')\biggr)
\end{equation}









\subsection{Subsampling Estimator}

The subsampling estimator is effective in its standard form when the heuristic satisfies  

\[
\text{heur}(b) = \mathbb{E}_{s \sim b} \text{heur}(s).
\]  

Many heuristics used in robotics planning, particularly in goal-directed tasks (e.g., an autonomous ground vehicle navigating to a goal), adhere to this property. While effective, the standard subsampling estimator is not truly unbiased due to dependencies between variables in the estimation process. To obtain a truly unbiased estimate, we must decompose the initial $Q$-value expression into independent subexpressions and estimate each component separately.  

The initial $Q$-value is defined as  

\begin{equation}
    Q_{\text{init}}(b, a) = \mathcal{C}(b, a) + \sum_{z \in Z} P(z| b, a) \cdot \text{heur}(b^z_a),
\end{equation}

where:
\begin{itemize}
    \item $\mathcal{C}(b, a)$ represents the cost of executing action $a$ in belief $b$.
    \item $P(z | b, a)$ is the probability of receiving observation $z$ after taking action $a$.
    \item $\text{heur}(b^z_a)$ is the heuristic value of the updated belief after observation $z$.
\end{itemize}

\subsubsection{Sources of Bias}

Directly estimating $Q_{\text{init}}(b, a)$ using the subsampling approach introduces bias due to the nonlinear dependencies within the summation. The observation probability $P(z | b, a)$ is given by:  

\begin{equation}
    P(z | b, a) = \sum_{s} b(s) \sum_{s'} T(s, a, s') O(s', a, z),
\end{equation}

which depends on the original belief distribution $b(s)$. Similarly, the heuristic term $\text{heur}(b^z_a)$ expands as:

\begin{equation}
    \text{heur}(b^z_a) = \sum_s b^z_a(s) \text{heur}(s),
\end{equation}

where  

\begin{equation}
    b^z_a(s) = \frac{\sum_{s'} T(s', a, s) O(s, a, z) b(s')}{P(z | b, a)}.
\end{equation}

Both the numerator and denominator of $b^z_a(s)$ are functions of $b(s)$, leading to statistical dependencies that introduce bias when estimated jointly.

\subsubsection{Unbiased Estimation}

To eliminate this bias, we construct independent estimates for three key components:
\begin{enumerate}
    \item The observation probability $P(z | b, a)$.
    \item The numerator of $\text{heur}(b^z_a)$.
    \item The denominator of $\text{heur}(b^z_a)$.
\end{enumerate}

Each component is estimated separately using independent samples:

\begin{enumerate}
    \item \textbf{Estimating $P(z | b, a)$}: Sample $s_1 \sim b(s)$ and compute:  
    \begin{equation}
        \hat{P}(z | b, a) = \mathbb{E}_{s_1 \sim b(s)} \sum_{s'} T(s_1, a, s') O(s', a, z).
    \end{equation}

    \item \textbf{Estimating the numerator of $\text{heur}(b^z_a)$}: Sample $s_2 \sim b(s)$ and compute:  
    \begin{equation}
        \hat{N} = \mathbb{E}_{s_2 \sim b(s)} \sum_s T(s_2, a, s) O(s, a, z) \text{heur}(s).
    \end{equation}

    \item \textbf{Estimating the denominator of $\text{heur}(b^z_a)$}: Sample $s_3 \sim b(s)$ and compute:  
    \begin{equation}
        \hat{D} = \mathbb{E}_{s_3 \sim b(s)} \sum_{s'} T(s_3, a, s') O(s', a, z).
    \end{equation}
\end{enumerate}

Using these independent estimates, the unbiased estimate of $Q_{\text{init}}(b, a)$ is given by:

\begin{equation}
    \hat{Q}_{\text{init}}(b, a) = \mathcal{\hat{C}}(b, a) + \sum_{z \in Z} \hat{P}(z | b, a) \frac{\hat{N}}{\hat{D}}.
\end{equation}

$\mathcal{\hat{C}}(b, a)$ can be estimated using all samples as linearity of expectation holds between $\mathcal{\hat{C}}(b, a)$ and the other estimated quantities.

\subsubsection{Practical Considerations}

Ensuring that $s_1, s_2,$ and $s_3$ are sampled independently eliminates the statistical dependency that caused bias in the subsampling estimator. However, this approach introduces a trade-off: while the estimate is now unbiased, it comes at the cost of increased variance. Given a fixed sampling budget $B$, we must distribute our available samples among three separate estimations, thereby increasing variance in each individual estimate.

In our evaluation, we employed the vanilla subsampling estimator, which performed well in practice. However, if an unbiased estimate of $Q_{\text{init}}(b, a)$ is critical, the outlined method provides a principled approach to achieving it.

\onecolumn
\newpage
\twocolumn
\bibliography{aaai25}

\end{document}